\documentclass[11pt]{article}

\usepackage[margin=1in]{geometry}
\usepackage{amsmath,amssymb,amsthm}
\usepackage{booktabs}
\usepackage{enumitem}
\usepackage{float}
\usepackage{graphicx}
\usepackage{hyperref}
\usepackage{natbib}

\newtheorem{theorem}{Theorem}
\newtheorem{lemma}{Lemma}
\newtheorem{assumption}{Assumption}
\newtheorem{definition}{Definition}
\newtheorem{proposition}{Proposition}

\title{ToolChain-CRC: Conformal Risk Control for Agentic AI Under Retrieval and Tool-Use Drift}

\author{
Jeffery Opoku\\
The University of Texas Rio Grande Valley, Edinburg, TX, USA\\
\texttt{jeffery.opoku01@utrgv.edu}
\and
David Banahene\\
Florida International University, Miami, FL, USA\\
\texttt{abanahene54@gmail.com}
}
\date{}

\begin{document}

\maketitle

\begin{abstract}
Modern AI agents retrieve documents, call tools, check intermediate information, and then produce a final answer or action. This creates a risk-control problem that is not visible from the final answer alone. A final response may look acceptable even when the retrieval was weak, a tool output was wrong, or an earlier step was unsupported. We propose ToolChain-CRC, a conformal risk-control method for retrieval-augmented and tool-using agents under drift. The method treats each agent run as a full trajectory of actions, observations, and final output. It builds step-level risk scores, combines them into a trajectory risk score, calibrates an accept-or-intervene rule, and adds an anytime alarm that can stop risky runs before the final answer. We prove trajectory-level risk control under exchangeable calibration runs, give a drift-aware extension with auditable constants, and prove an anytime escalation rule through a supermartingale construction. Experiments cover synthetic tool-chain drift, RAG/tool-use stress tests, public SQuAD-derived retrieval tasks, an API-free agentic QA case study, ablations, target-risk sensitivity checks, 20-seed robustness checks, a drift-margin audit, and a live RAG/tool-use agent benchmark. Across these settings, final-answer-only calibration can miss retrieval and tool failures, while trajectory-level calibration keeps accepted-trajectory risk below the target.
\end{abstract}



\section{Introduction}

Modern AI systems do not just answer questions. They look things up, call tools, check intermediate information, and then write a final response. This is useful, but it also creates a new kind of risk. The system can go wrong before the final answer is ever written. It may retrieve the wrong document. It may trust a bad tool output. It may miss a conflict between sources. Or it may combine weak pieces of evidence into an answer that sounds confident.

Final-answer calibration is therefore often too narrow. If we only check the final response, we may miss the earlier steps that made the response unsafe. This is not only a practical problem. It is also a statistical problem: a score that only sees the final answer cannot, by itself, certify failures that happened upstream and are not visible in that final score. This paper studies how to control risk over the whole path an AI agent takes.

The main question is:

\begin{quote}
How can we calibrate the risk of the whole agent trajectory, not only the final answer?
\end{quote}

Conformal prediction gives prediction sets with finite-sample coverage under exchangeability \citep{vovk2005algorithmic,angelopoulos2021gentle}. Conformal risk control extends this idea to user-chosen risks \citep{angelopoulos2024conformal}. These tools are attractive because they can be placed around complicated black-box systems. But many agentic systems are no longer a single black box. They are chains. Retrieval-augmented generation uses outside evidence \citep{lewis2020retrieval}. Tool-using language systems rely on outside calls and intermediate observations \citep{lymperopoulos2025tools}. Recent work has also studied anytime-valid conformal risk control \citep{hultberg2026anytime}, selective acting for language-model deployment \citep{khosravi2026selective}, and conformal safety wrappers for planning \citep{doula2025safepath}.

The challenge is that these agents are sequential. The second step depends on the first step. The final answer depends on retrieved context and tool outputs. When the retrieval system, tool behavior, or user population changes, old calibration data may no longer describe the current agent.

We propose ToolChain-CRC. It is a conformal risk-control layer for tool-using and retrieval-augmented agents. The key idea is simple: treat each agent run as a full trajectory. The trajectory includes the prompt, retrieval steps, tool calls, observations, intermediate checks, and final answer. ToolChain-CRC scores risk at each step, combines those scores into a trajectory risk, and calibrates an accept-or-intervene rule. It can also warn while the agent is still acting, before the final answer is produced.

The method is meant to be practical. A developer may not be able to prove that a large agent is always correct. But the developer can collect past agent runs, score failures, and ask for a rule: continue, retrieve again, call a safer tool, abstain, or send the case to a human. ToolChain-CRC gives that rule together with diagnostics that say where the risk is coming from.

\paragraph{Novelty claim.}
The central claim of this paper is not that conformal risk control is new, or that uncertainty for language models is new. The claim is that the right statistical object for a tool-using agent is the \emph{whole trajectory}. ToolChain-CRC calibrates risk over the sequence of retrieved evidence, tool calls, observations, intermediate decisions, and final output. This is different from final-answer calibration, token-level calibration, and ordinary prompt-response calibration. It also makes the diagnostics useful: the method can say whether risk is coming from retrieval, tool use, synthesis, or drift away from the calibration trajectories.

\medskip
\noindent\fbox{%
\begin{minipage}{0.96\textwidth}
\textbf{Main contributions.}
\begin{enumerate}[label=(\roman*)]
    \item We make the full agent trajectory the calibrated statistical object, rather than only the final answer.
    \item We define concrete step-risk scores, audited step labels, and accepted-trajectory losses that combine retrieval risk, tool-use risk, evidence-support risk, and final-answer risk.
    \item We show why final-answer-only calibration can have a hidden upstream-risk blind spot, then give trajectory-level risk-control guarantees, a drift-aware extension, and an anytime escalation rule.
    \item We report practical diagnostics: source of risk, drift score, effective trajectory sample size, intervention rate, and utility-aware policy choice.
    \item We test the idea using synthetic drift, RAG/tool-use stress tests, a public SQuAD-derived experiment, an agentic QA case study, component ablations, sensitivity checks, 20-seed robustness checks, a drift-margin audit, and a live RAG/tool-use agent benchmark.
\end{enumerate}
\end{minipage}}
\medskip

\begin{figure}[t]
    \centering
    \includegraphics[width=0.98\textwidth]{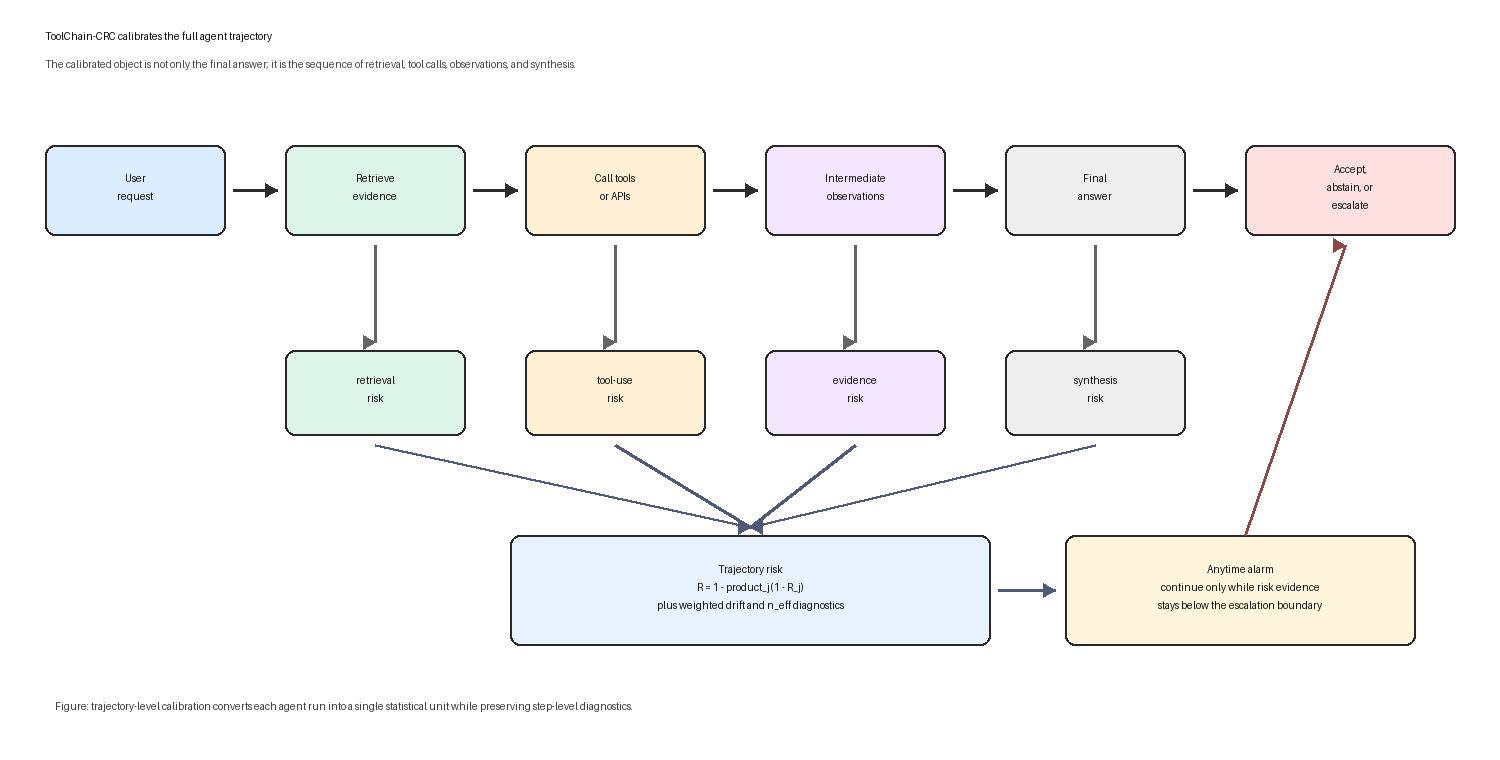}
    \caption{ToolChain-CRC calibrates the complete agent trajectory. The method keeps the trajectory as the statistical unit while preserving step-level diagnostics for retrieval risk, tool-use risk, evidence risk, synthesis risk, drift, effective calibration support, and anytime escalation.}
    \label{fig:architecture}
\end{figure}

\section{Positioning Against Related Work}

This paper sits between conformal risk control, anytime validity, retrieval-augmented generation, and tool-using AI agents.

\paragraph{Conformal risk control.}
Conformal risk control gives a distribution-free way to control losses chosen by the user \citep{angelopoulos2024conformal}. The usual setup calibrates a decision rule on examples that can be treated as exchangeable. ToolChain-CRC keeps the same risk-control goal, but changes the unit of calibration from one example to one full agent trajectory.

\paragraph{Anytime-valid risk control.}
Anytime-valid conformal risk control studies guarantees that remain valid as data arrive over time \citep{hultberg2026anytime}. Recent selective-acting work also studies language-model deployment with anytime risk certificates \citep{khosravi2026selective}. ToolChain-CRC is complementary. Its focus is not only validity over time, but also breaking risk into retrieval, tool use, and final synthesis inside an adaptive agent run.

\paragraph{Tool use and retrieval.}
Uncertainty for tool-using language-model systems is becoming its own problem because both the model and the outside tool can be wrong \citep{lymperopoulos2025tools}. Retrieval-augmented systems add another layer because the retrieved evidence may be stale, irrelevant, or misleading \citep{lewis2020retrieval}. ToolChain-CRC treats retrieval and tool behavior as part of the risk object, not as hidden preprocessing.

\paragraph{Safety wrappers for agents.}
Conformal wrappers have been proposed for specific agent or planning settings, such as safe LLM-based autonomous navigation \citep{doula2025safepath}. The goal here is broader: to define a reusable statistical wrapper for any agent whose behavior can be written as a sequence of actions, observations, and final outputs.

\begin{table}[H]
\centering
\small
\caption{Positioning of ToolChain-CRC relative to nearby research areas. The table is qualitative and is meant to clarify the statistical target of the paper.}
\label{tab:positioning}
\begin{tabular}{p{0.25\textwidth}cccccc}
\toprule
Approach & Final & Traj. & Ret. drift & Tool risk & Alarm & \(n_{\mathrm{eff}}\) \\
\midrule
Standard conformal risk control & Yes & No & No & No & No & No \\
Weighted conformal prediction & Yes & No & Partial & No & No & Sometimes \\
Anytime CRC & Yes & No & No & No & Yes & No \\
Tool-use uncertainty & Yes & No & Partial & Yes & No & No \\
Safe agent wrappers & Task-specific & Partial & No & Partial & Yes & No \\
ToolChain-CRC & Yes & Yes & Yes & Yes & Yes & Yes \\
\bottomrule
\end{tabular}
\end{table}

Table~\ref{tab:positioning} shows the intended contribution. ToolChain-CRC is not meant to replace conformal risk control or anytime-valid risk control. It uses those ideas, but changes the calibrated object from a single response to a full agent trajectory. It also reports diagnostics that matter for retrieval and tool-use deployment.

\section{Setup}

Let \(X_i\) be the user request for agent run \(i\). The agent produces a trajectory
\[
    \tau_i = (X_i, A_{i,1}, O_{i,1}, A_{i,2}, O_{i,2}, \ldots, A_{i,T_i}, O_{i,T_i}, Y_i),
\]
where \(A_{i,j}\) is the \(j\)-th action or tool call, \(O_{i,j}\) is the returned observation, \(T_i\) is the trajectory length, and \(Y_i\) is the final response.

For each step \(j\), we keep two quantities. The first is a raw step-risk score
\[
    S_{i,j}\in[0,1],
\]
where larger values mean the step looks less trustworthy. The second is an audited step label
\[
    Y_{i,j}\in[0,1],
\]
where larger values mean the step actually failed under the audit rule. The raw score is available when the agent is running. The audited label is used later for calibration, evaluation, and reporting.

In a retrieval step, \(S_{i,j}\) can be one minus retrieval confidence, one minus source-support confidence, or a verifier score saying that the retrieved context does not support the question. In a tool step, \(S_{i,j}\) can come from a tool-error flag, an invalid-output flag, disagreement across tools, or a domain verifier. In a synthesis step, \(S_{i,j}\) can be an unsupported-answer score, a factuality score, or a task-specific loss proxy. The corresponding \(Y_{i,j}\) values come from human audit, gold labels, deterministic validators, or a frozen verifier. In the experiments below, the synthetic studies use known retrieval, tool, and synthesis failure indicators; the SQuAD-derived study uses whether the retrieved passage supports the answer; and the agentic QA study uses scripted validators for retrieval support and final answer support.

We combine the raw step scores into a single trajectory score
\[
    S_i^{\mathrm{traj}}
    =
    1-\prod_{j=1}^{T_i}\{1-S_{i,j}\}.
\]
This noisy-or form is large when any important step looks risky. A weighted additive score,
\[
    S_i^{\mathrm{traj}}
    =
    \sum_{j=1}^{T_i} \omega_{i,j} S_{i,j},
    \qquad
    \omega_{i,j}\geq 0,\quad \sum_{j=1}^{T_i}\omega_{i,j}=1,
\]
can be used when a system owner wants to put more weight on some steps than others. The same construction gives an audited trajectory failure label
\[
    Y_i^{\mathrm{traj}}
    =
    1-\prod_{j=1}^{T_i}\{1-Y_{i,j}\}.
\]

The policy threshold \(\lambda\) controls whether the agent accepts the trajectory. In this paper, the default accept rule is
\[
    A_i(\lambda)=\mathbf{1}\{S_i^{\mathrm{traj}}\leq \lambda\}.
\]
The calibration loss is therefore the accepted-trajectory loss
\[
    R_i^{\mathrm{traj}}(\lambda)
    =
    Y_i^{\mathrm{traj}} A_i(\lambda).
\]
This is the exact quantity controlled by conformal risk control. It is zero when the system intervenes or abstains, and it is the audited trajectory failure when the system accepts.

The target is to choose \(\lambda\) so that
\[
    \mathbb{E}\{R^{\mathrm{traj}}(\lambda)\}\leq \alpha,
\]
where \(\alpha\) is the user-chosen accepted-risk level.

\paragraph{Why the trajectory view matters.}
The actions inside a trajectory may be highly dependent. A bad retrieval step can cause a bad tool call, and a bad tool call can cause a bad final answer. ToolChain-CRC does not require the steps inside one trajectory to be independent. It treats the entire run \(\tau_i\) as one statistical object. This is the key modeling choice. We allow complicated dependence within a run, but ask for calibration information across runs.

\begin{proposition}[Final-answer scores can hide upstream risk]
Let \(A\in\{0,1\}\) be the event that the system accepts an answer, let \(V\in\{0,1\}\) be the visible final-answer failure indicator, and let \(U\in\{0,1\}\) be an upstream failure indicator such as unsupported retrieval or a wrong tool output. The full trajectory failure is
\[
    R^{\mathrm{traj}}=\mathbf{1}\{U=1\ \text{or}\ V=1\}.
\]
If the accept rule \(A\) depends only on a final-answer score, then controlling the final visible risk \(\mathbb{E}[A V]\) does not by itself control the trajectory risk \(\mathbb{E}[A R^{\mathrm{traj}}]\). In fact,
\[
    \mathbb{E}[A R^{\mathrm{traj}}]
    =
    \mathbb{E}[A V]
    +
    \mathbb{E}[A(1-V)U].
\]
The second term is upstream risk that is invisible to a final-answer-only score unless extra assumptions connect \(U\) to that score.
\end{proposition}

\begin{proof}
Since \(R^{\mathrm{traj}}=\mathbf{1}\{U=1\ \text{or}\ V=1\}=V+(1-V)U\), multiplying by the accept indicator \(A\) and taking expectations gives
\[
    \mathbb{E}[A R^{\mathrm{traj}}]
    =
    \mathbb{E}[A V]
    +
    \mathbb{E}[A(1-V)U].
\]
The first term is the risk seen by final-answer calibration. The second term is accepted upstream failure among cases whose final visible failure indicator is zero. This term can be positive even when \(\mathbb{E}[A V]\) is small. Therefore final-answer-only control is not enough to control the full trajectory risk without an additional assumption that upstream failures are captured by the final-answer score.
\end{proof}

This proposition is the paper's main statistical motivation. It is not saying that final-answer scores are useless. They are useful for final-answer errors. The point is narrower and sharper: when the target risk includes retrieval, tool use, evidence support, and synthesis, the calibrated score must see those parts of the run.

\begin{definition}[Trajectory policy]
A trajectory policy \(\pi_\lambda\) is a rule that maps the partial history
\[
    H_{i,j}=(X_i,A_{i,1},O_{i,1},\ldots,A_{i,j-1},O_{i,j-1})
\]
to the next action \(A_{i,j}\), an escalation decision, or a final answer. In the default accept-threshold version, \(\lambda\) is the largest trajectory score the system is willing to accept. Smaller values are more cautious because fewer trajectories are accepted.
\end{definition}

\begin{assumption}[Monotone accept-threshold policy]
For every realized trajectory \(\tau\), the accepted-trajectory loss \(R^{\mathrm{traj}}(\lambda)\in[0,1]\) increases, or stays the same, as \(\lambda\) increases. In words, accepting more trajectories cannot lower the measured accepted-trajectory loss for the same underlying run.
\end{assumption}

This monotonicity condition is natural when \(\lambda\) controls whether the agent accepts, abstains, retrieves again, or escalates. It is also useful computationally because a one-dimensional grid of candidate policies can be searched.

\section{The ToolChain-CRC Procedure}

The method has a calibration phase and a deployment phase. In calibration, completed agent trajectories are scored and used to select a risk threshold. In deployment, the same scoring rule is applied online while the agent acts.

\medskip
\noindent\fbox{%
\begin{minipage}{0.96\textwidth}
\textbf{Algorithm 1: ToolChain-CRC calibration and deployment}

\begin{enumerate}[leftmargin=1.5em]
    \item Collect calibration trajectories
    \[
        \tau_i=(X_i,A_{i,1},O_{i,1},\ldots,A_{i,T_i},O_{i,T_i},Y_i),
        \qquad i=1,\ldots,n.
    \]
    \item For each trajectory, compute raw step scores \(S_{i,j}\) and audited step labels \(Y_{i,j}\) for retrieval, tool use, evidence support, and final synthesis.
    \item Combine the step scores into a trajectory score \(S_i^{\mathrm{traj}}\), for example
    \[
        S_i^{\mathrm{traj}}=1-\prod_{j=1}^{T_i}\{1-S_{i,j}\}.
    \]
    \item For each candidate \(\lambda\), form \(A_i(\lambda)=\mathbf{1}\{S_i^{\mathrm{traj}}\leq\lambda\}\) and \(R_i^{\mathrm{traj}}(\lambda)=Y_i^{\mathrm{traj}}A_i(\lambda)\). Choose the largest acceptable policy threshold \(\widehat{\lambda}\) whose corrected calibration risk is below the target \(\alpha\).
    \item During deployment, update trajectory features, drift score, effective sample size, and the anytime alarm after each retrieval or tool-use step.
    \item If the trajectory score or anytime alarm crosses its boundary, abstain, retrieve again, call a safer tool, or escalate to human review. Otherwise return the final answer.
\end{enumerate}
\end{minipage}}

\medskip

The algorithm is intentionally modular. The raw score can be factuality error, toxicity, weak source support, wrong-tool risk, unsafe-action risk, or a task-specific score. The conformal layer only requires that the accepted-trajectory loss can be evaluated on calibration trajectories.

\paragraph{Choosing a useful operating point.}
Risk control alone is not enough in a real system. A rule that sends every case to a human may be safe, but it is not useful. After computing the feasible thresholds, we therefore choose the useful threshold among the safe ones. Let
\[
    \widehat{\Lambda}_{\alpha}
    =
    \left\{\lambda\in\Lambda:
    \frac{1}{n+1}\left(1+\sum_{i=1}^{n}R_i^{\mathrm{traj}}(\lambda)\right)
    \leq \alpha
    \right\}.
\]
For any cost weight \(\eta\geq 0\), choose
\[
    \widehat{\lambda}_{\eta}
    =
    \arg\min_{\lambda\in\widehat{\Lambda}_{\alpha}}
    \left\{
    \widehat{R}^{\mathrm{traj}}(\lambda)
    +
    \eta\,\widehat{I}(\lambda)
    \right\},
\]
where \(\widehat{I}(\lambda)\) is the calibration intervention rate. Small \(\eta\) favors lower empirical trajectory risk. Larger \(\eta\) favors fewer abstentions, repeated retrievals, or human reviews. This gives practitioners a simple way to pick an operating point after the risk guarantee has already been enforced.

\section{Split Trajectory Risk Control}

Suppose we have calibration trajectories
\[
    \tau_1,\ldots,\tau_n
\]
that are exchangeable with a future trajectory \(\tau_{n+1}\). For a grid of candidate thresholds \(\Lambda\), compute calibration risks
\[
    R_i^{\mathrm{traj}}(\lambda), \qquad i=1,\ldots,n,\quad \lambda\in\Lambda.
\]

Choose
\[
    \widehat{\lambda}
    =
    \sup\left\{\lambda\in\Lambda:
    \frac{1}{n+1}\left(1+\sum_{i=1}^{n}R_i^{\mathrm{traj}}(\lambda)\right)
    \leq \alpha
    \right\}.
\]

The extra \(1/(n+1)\) term is a conservative finite-sample correction.

\begin{theorem}[Trajectory conformal risk control with adaptive tool calls]
Assume the calibration trajectories and the future trajectory are exchangeable as complete trajectories. The actions inside each trajectory may be adaptive and dependent on earlier observations in the same trajectory. Assume also that the monotone accept-threshold policy condition holds on the ordered grid \(\Lambda\). If \(R_i^{\mathrm{traj}}(\lambda)\in[0,1]\) and the rule above selects \(\widehat{\lambda}\), then the future trajectory satisfies
\[
    \mathbb{E}\{R_{n+1}^{\mathrm{traj}}(\widehat{\lambda})\}\leq \alpha
\]
up to the usual finite-grid discretization error from searching over \(\Lambda\).
\end{theorem}

\begin{proof}
This is the split conformal risk-control argument applied to the trajectory-level loss. The only change is the unit of analysis. Instead of applying the argument to a single prompt-response pair, we apply it to the full trajectory \(\tau_i\). Dependence among actions inside \(\tau_i\) is allowed because \(R_i^{\mathrm{traj}}(\lambda)\) is a measurable function of the complete trajectory. Exchangeability is required only across complete trajectories.

For a fixed \(\lambda\), exchangeability gives the same average risk for each of the \(n+1\) calibration-plus-test trajectories. The conservative correction \(1/(n+1)\) is the standard split-CRC correction for bounded losses in \([0,1]\). The monotonicity assumption makes the candidate policies nested as \(\lambda\) increases, so the largest feasible threshold selected by the empirical rule inherits the split-CRC risk guarantee. The finite grid only adds the usual approximation error from restricting the search to \(\Lambda\).
\end{proof}

\begin{proposition}[Utility-aware selection preserves feasibility]
Suppose \(\widehat{\lambda}_{\eta}\in\widehat{\Lambda}_{\alpha}\) is selected by the utility-aware rule above. Under the monotone accept-threshold policy condition, \(\widehat{\lambda}_{\eta}\) is at least as cautious as the largest feasible split-CRC threshold \(\widehat{\lambda}\). Consequently,
\[
    R_{n+1}^{\mathrm{traj}}(\widehat{\lambda}_{\eta})
    \leq
    R_{n+1}^{\mathrm{traj}}(\widehat{\lambda})
\]
for every future trajectory, and the same trajectory risk-control guarantee applies to \(\widehat{\lambda}_{\eta}\).
\end{proposition}

\begin{proof}
Because accepted-trajectory loss is nondecreasing in the accept threshold \(\lambda\), the feasible set \(\widehat{\Lambda}_{\alpha}\) is a lower set on the ordered grid. Its largest element is the split-CRC threshold \(\widehat{\lambda}\). Any utility-aware choice inside this feasible set is therefore at least as cautious as \(\widehat{\lambda}\). The pointwise monotonicity of the trajectory risk gives
\[
    R_{n+1}^{\mathrm{traj}}(\widehat{\lambda}_{\eta})
    \leq
    R_{n+1}^{\mathrm{traj}}(\widehat{\lambda}).
\]
Taking expectations and applying the previous theorem gives the result.
\end{proof}

\begin{proposition}[Source-of-risk decomposition]
For the noisy-or trajectory risk
\[
    R^{\mathrm{traj}}=1-\prod_{j=1}^{T}(1-R_j),
\]
where \(R_j\in[0,1]\), the following bounds hold:
\[
    \max_{1\leq j\leq T}R_j
    \leq
    R^{\mathrm{traj}}
    \leq
    \sum_{j=1}^{T}R_j .
\]
Moreover, if all step risks are small, then
\[
    R^{\mathrm{traj}}
    =
    \sum_{j=1}^{T}R_j
    -
    \sum_{j<k}R_jR_k
    +
    O\left(\sum_{j<k<\ell}R_jR_kR_\ell\right).
\]
\end{proposition}

\begin{proof}
The lower bound follows because \(\prod_j(1-R_j)\leq 1-R_k\) for every \(k\), so \(1-\prod_j(1-R_j)\geq R_k\). The upper bound is the union bound, which follows by expanding the product or by induction. The expansion is the inclusion-exclusion formula for \(1-\prod_j(1-R_j)\).
\end{proof}

This proposition gives a simple diagnostic interpretation. A large trajectory risk can come from one severe step or from several moderate risks that compound. This is why final-answer-only calibration can miss failures: the final answer may look acceptable even when retrieval or tool-use risk has already accumulated.

\section{Weighted Tool-Use Drift}

Exchangeability is often too strong for real agent systems. Retrieval corpora change, APIs change, user questions change, and the agent may learn new tool-use patterns. To handle this, define a trajectory representation
\[
    \phi(\tau_i)
\]
that summarizes the request, retrieved evidence, tool calls, and final response. Let \(z_t=\phi(\tau_t)\). For a current trajectory or window of trajectories, assign calibration weights
\[
    w_i(t)
    =
    \frac{\exp\{-d(z_i,z_t)/h\}}
    {\sum_{k=1}^{n}\exp\{-d(z_k,z_t)/h\}},
\]
where \(d\) measures trajectory similarity and \(h>0\) is a bandwidth.

The weighted effective sample size is
\[
    n_{\mathrm{eff}}(t)
    =
    \frac{\left(\sum_{i=1}^{n}w_i(t)\right)^2}
    {\sum_{i=1}^{n}w_i(t)^2}.
\]

Small \(n_{\mathrm{eff}}\) means the system is relying on too few calibration trajectories, so the conformal threshold should be treated as fragile.

\begin{assumption}[Lipschitz drift of trajectory risk]
Let \(P_t\) be the current trajectory distribution and let \(P_{w,t}\) be the weighted calibration distribution induced by the weights \(w_i(t)\). For the selected policy class, suppose
\[
    \left|
    \mathbb{E}_{P_t}\{R^{\mathrm{traj}}(\lambda)\}
    -
    \mathbb{E}_{P_{w,t}}\{R^{\mathrm{traj}}(\lambda)\}
    \right|
    \leq
    L\,W_1(P_t,P_{w,t})
\]
for all candidate \(\lambda\), where \(W_1\) is a Wasserstein-1 discrepancy on trajectory representations.
\end{assumption}

\begin{theorem}[Approximate risk control under tool-use drift]
Fix a deployment time \(t\), and let \(m=|\Lambda|\) be the number of candidate thresholds. Conditional on the trajectory weights \(w_i(t)\), suppose the weighted calibration rule selects \(\widehat{\lambda}_t\) so that the weighted empirical trajectory risk is at most
\[
    \alpha-\epsilon_t,
\]
where
\[
    \epsilon_t
    =
    \sqrt{\frac{\log(m/\delta)}{2n_{\mathrm{eff}}(t)}} .
\]
Under the Lipschitz drift assumption, with probability at least \(1-\delta\) over the weighted calibration sample,
\[
    \mathbb{E}_{P_t}\{R^{\mathrm{traj}}(\widehat{\lambda}_t)\}
    \leq
    \alpha
    +
    L\,W_1(P_t,P_{w,t}).
\]
\end{theorem}

\begin{proof}
For any fixed \(\lambda\), the weighted empirical risk is a weighted average of bounded losses in \([0,1]\). Conditional on the weights, Hoeffding's inequality for weighted bounded averages gives concentration at rate \(n_{\mathrm{eff}}(t)=1/\sum_i w_i(t)^2\). A union bound over the \(m\) candidate thresholds gives the correction \(\epsilon_t\). Therefore, with probability at least \(1-\delta\),
\[
    \mathbb{E}_{P_{w,t}}\{R^{\mathrm{traj}}(\widehat{\lambda}_t)\}
    \leq
    \alpha .
\]
The Lipschitz drift assumption then transfers this guarantee from the weighted calibration distribution \(P_{w,t}\) to the current trajectory distribution \(P_t\), adding the discrepancy term \(L W_1(P_t,P_{w,t})\).
\end{proof}

The theorem separates three practical quantities. The target \(\alpha\) is chosen by the user. The effective sample size term tells whether the calibration evidence is strong enough. The Wasserstein term measures whether the current agent behavior has drifted away from the weighted calibration trajectories.

\paragraph{How to read the drift guarantee.}
The drift result is deliberately stated as an approximate guarantee, not as a fully distribution-free guarantee. The exact finite-sample guarantee comes from exchangeability of complete trajectories. Once the deployment distribution drifts, no method can certify the same risk level without some assumption connecting the calibration distribution to the current distribution. Here that connection is the Lipschitz drift assumption in the chosen trajectory representation. The bound is useful when nearby trajectories in representation space have similar trajectory risk. It is less useful if the representation misses an important failure source, if the weights are supported by only a few calibration examples, or if the current system has moved into a genuinely new regime.

For this reason, ToolChain-CRC reports \(n_{\mathrm{eff}}\), drift score, and intervention rate together with the calibrated decision. A small \(n_{\mathrm{eff}}\) or a large drift score should not be hidden; it is evidence that the system is leaving the region where the calibration data can support strong claims. In that case, the statistically honest action is to abstain more often, collect new calibration trajectories, or route the case to human review.

\paragraph{How the drift constants are estimated.}
In practice, the drift pieces are not known exactly: \(P_t\), \(P_{w,t}\), and \(L\) must be estimated. ToolChain-CRC therefore treats the drift bound as an auditable reporting tool, not as a magic certificate. Let \(z_i=\phi(\tau_i)\) be the trajectory representation for calibration run \(i\), and let \(\bar z_t\) be the mean representation in the current deployment window. We report the empirical drift proxy
\[
    \widehat D_t
    =
    \left\|
    \bar z_t-\sum_{i=1}^n w_i(t)z_i
    \right\|_2 ,
\]
together with \(n_{\mathrm{eff}}(t)\). This proxy is simple, reproducible, and easy to audit. A richer implementation may replace it with an optimal-transport or nearest-neighbor distance, but the reporting logic is the same.

The constant \(L\) is also estimated rather than assumed. On held-out calibration trajectories, compute pairwise slopes
\[
    \ell_{i\ell}
    =
    \frac{
    \left|Y_i^{\mathrm{traj}}-Y_{\ell}^{\mathrm{traj}}\right|
    }{
    \|z_i-z_{\ell}\|_2+\varepsilon
    },
    \qquad i\neq \ell ,
\]
where \(\varepsilon>0\) avoids division by zero. We use a high empirical quantile, for example the 90th or 95th percentile, as a conservative estimate \(\widehat L\). The operational drift margin is then
\[
    \widehat\Delta_t(\widehat L)
    =
    \widehat L\,\widehat D_t
    +
    \sqrt{\frac{\log(m/\delta)}{2n_{\mathrm{eff}}(t)}} .
\]
When this margin is small, the weighted calibration evidence is close to the current deployment window. When it is large, the right action is not to pretend the guarantee is strong; the system should abstain more, collect fresh calibration trajectories, or report that the current window is outside the supported region. We also recommend reporting a sensitivity curve over several values of \(L\), because reviewers and users can then see whether the conclusion depends on a fragile constant.

\section{Diagnostics}

ToolChain-CRC is designed to return more than a binary accept-or-abstain decision. Its diagnostics are part of the method because they explain why a trajectory is risky.

\paragraph{Retrieval contribution.}
Let \(R_{\mathrm{ret}}\) measure whether the retrieved evidence is irrelevant, stale, contradictory, or unsupported. A high retrieval contribution means the agent is likely to fail before any reasoning or synthesis begins.

\paragraph{Tool-use contribution.}
Let \(R_{\mathrm{tool}}\) measure wrong tool selection, noisy tool output, failed API calls, or disagreement across tools. This is important because a final answer may look fluent while depending on an unreliable tool output.

\paragraph{Synthesis contribution.}
Let \(R_{\mathrm{syn}}\) measure whether the final answer is unsupported by the retrieved evidence and tool observations. This is closest to ordinary final-answer calibration, but here it is only one part of the trajectory risk.

\paragraph{Trajectory drift.}
The drift score compares current trajectory features with weighted calibration trajectories. It can be computed on prompt embeddings, retrieval features, tool-call patterns, observation summaries, or final-answer features. A large value means the current agent run is outside the region where calibration evidence is strong.

\paragraph{Effective trajectory sample size.}
The effective sample size
\[
    n_{\mathrm{eff}}=\frac{(\sum_i w_i)^2}{\sum_i w_i^2}
\]
reports how many calibration trajectories effectively support the current decision. A low value is a warning that the conformal threshold may be fragile even if the empirical risk looks acceptable.

\paragraph{Intervention rate.}
The intervention rate records how often the agent abstains, retrieves again, calls a safer tool, or asks for human review. This makes the method auditable: a system that controls risk only by stopping almost everything is not useful, while a system that rarely intervenes may be unsafe under drift.

\section{Anytime Risk Alarm}

An agent should not always wait until the final answer to detect risk. Let \(\mathcal F_s\) be the information available after step or run \(s\). The threshold \(\lambda\) is chosen from calibration data before deployment, so during deployment it is fixed, or more generally \(\mathcal F_0\)-measurable. Define the excess risk signal
\[
    E_s(\lambda)=R_s(\lambda)-\alpha .
\]
Positive values mean that the current step or run is contributing more risk than the target allows.

For a predictable betting parameter \(\eta_s\geq 0\), define
\[
    M_t(\lambda)
    =
    \prod_{s=1}^{t}
    \exp\{\eta_s E_s(\lambda)-\psi_s(\eta_s)\},
\]
where \(\psi_s(\eta_s)\) is chosen before observing step \(s\). The role of \(\psi_s\) is to make each exponential factor have conditional expectation at most one under acceptable operation. This is the standard test-supermartingale idea behind safe anytime-valid inference and e-processes \citep{howard2021timeuniform,ramdas2023game}.

\begin{lemma}[Supermartingale alarm construction]
Suppose that under the acceptable-risk regime, for every deployment step \(s\),
\[
    \mathbb E\!\left[
        \exp\{\eta_s E_s(\lambda)-\psi_s(\eta_s)\}
        \,\middle|\, \mathcal F_{s-1}
    \right]\leq 1,
\]
where \(\eta_s\) and \(\psi_s(\eta_s)\) are \(\mathcal F_{s-1}\)-measurable. Then \(M_t(\lambda)\) is a nonnegative supermartingale with \(M_0(\lambda)=1\).
\end{lemma}

\begin{proof}
The process is nonnegative by construction and \(M_0(\lambda)=1\). Since \(M_{t-1}(\lambda)\) is \(\mathcal F_{t-1}\)-measurable,
\[
\begin{aligned}
    \mathbb E\!\left[M_t(\lambda)\mid \mathcal F_{t-1}\right]
    &=
    M_{t-1}(\lambda)
    \mathbb E\!\left[
        \exp\{\eta_t E_t(\lambda)-\psi_t(\eta_t)\}
        \,\middle|\, \mathcal F_{t-1}
    \right]  \\
    &\leq M_{t-1}(\lambda).
\end{aligned}
\]
Thus \(M_t(\lambda)\) is a supermartingale.
\end{proof}

When \(M_t(\lambda)\) becomes large, the system has evidence that realized risk is exceeding the target.

The resulting alarm rule is
\[
    \text{escalate at time }t
    \quad\text{if}\quad
    M_t(\widehat{\lambda})\geq \frac{1}{\delta}.
\]

This gives an anytime warning level \(\delta\): the agent can be stopped, routed to a human, or forced to retrieve again before it produces a risky final answer.

\begin{theorem}[Anytime escalation control]
Assume the conditions of the supermartingale alarm lemma. Then
\[
    \mathbb{P}\left\{\sup_{t\geq 1}M_t(\lambda)\geq \frac{1}{\delta}\right\}
    \leq
    \delta .
\]
Thus the escalation rule \(M_t(\lambda)\geq 1/\delta\) has false-alarm probability at most \(\delta\) under the acceptable-risk regime.
\end{theorem}

\begin{proof}
This is Ville's inequality for nonnegative supermartingales. Since \(M_0(\lambda)=1\), the probability that the process ever crosses \(1/\delta\) is at most \(\delta\).
\end{proof}

The guarantee is useful because the agent does not need to wait for a fixed horizon. The same alarm can be checked after retrieval, after a tool call, after synthesis, or after a batch of deployed agent runs.

\section{First Simulation}

We first test the idea in a synthetic tool-chain environment. Each agent trajectory has three possible failure points: retrieval, tool use, and final synthesis. The calibration data are generated from a stable environment. The test stream begins in the same environment and then shifts to a new regime where retrieval is less reliable, tool outputs are noisier, and tasks are more complex.

We compare three policies. The first policy calibrates only a final-answer score. The second policy calibrates a trajectory score but does not adapt to drift. The third policy is ToolChain-CRC, which uses the trajectory score with similarity weighting, drift monitoring, and an effective sample size correction.

Figure~\ref{fig:first-simulation} shows the main pattern. The final-answer-only method misses many upstream failures after drift because the final score does not fully see retrieval and tool errors. The fixed trajectory method is more conservative and controls risk better. ToolChain-CRC gives the lowest post-drift risk in this simulation because it recognizes that the current trajectories no longer look like the original calibration population.

\begin{figure}[t]
    \centering
    \includegraphics[width=0.98\textwidth]{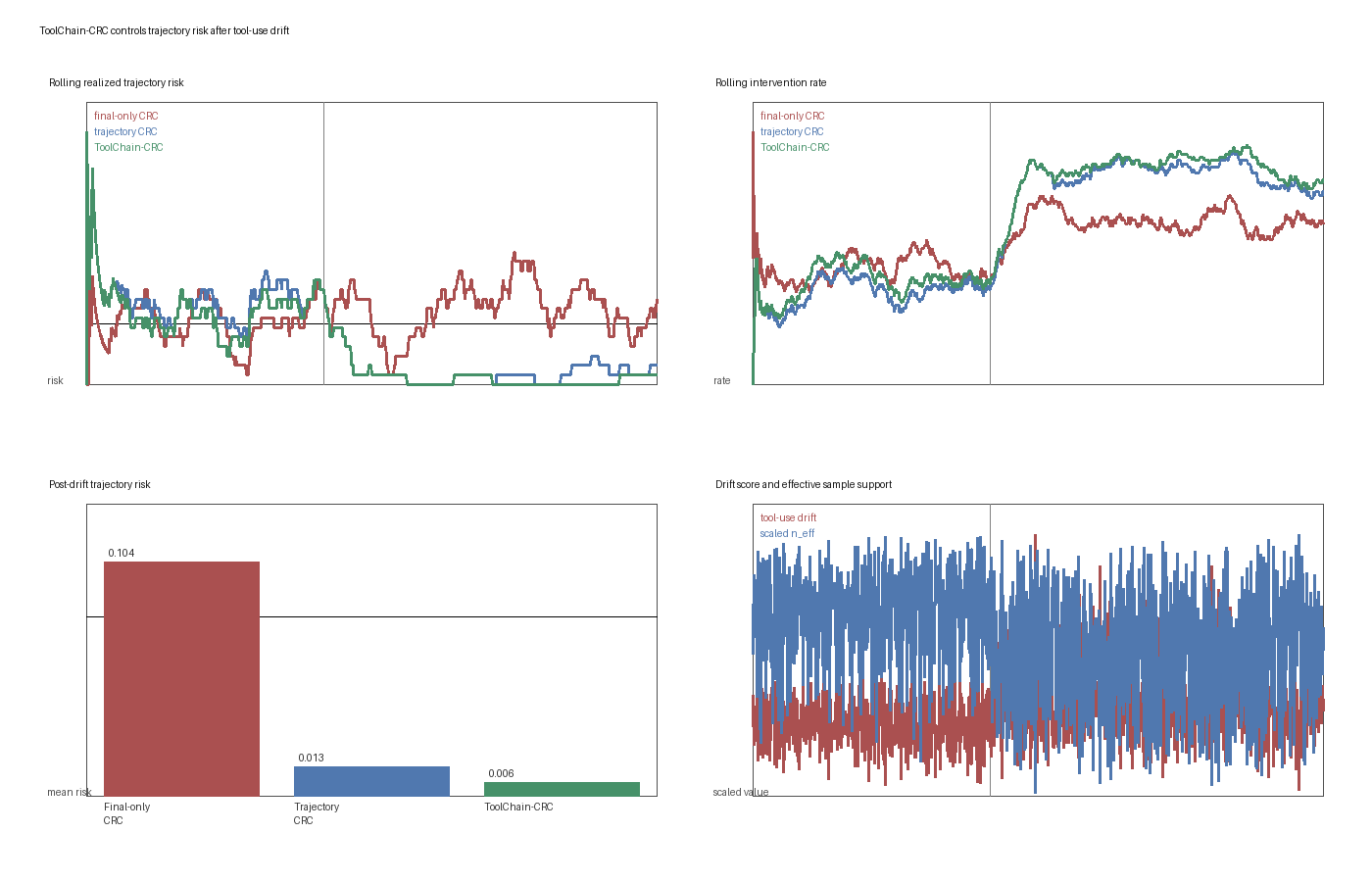}
    \caption{First ToolChain-CRC simulation. The stream starts in a stable regime and then enters a shifted regime with worse retrieval, noisier tool outputs, and harder tasks. The vertical line marks the start of drift. ToolChain-CRC uses trajectory-level risk and local calibration support, which reduces post-drift trajectory risk compared with final-answer-only calibration.}
    \label{fig:first-simulation}
\end{figure}

The post-drift realized trajectory risk is \(0.104\) for final-answer-only CRC, \(0.013\) for fixed trajectory CRC, and \(0.006\) for ToolChain-CRC. The corresponding overall intervention rates are \(0.568\), \(0.643\), and \(0.677\). This illustrates the main tradeoff: trajectory-level monitoring is more cautious, but it catches failures that single-response calibration can miss.

\section{Benchmark-Style RAG and Tool-Use Stress Test}

The first simulation isolates the basic mechanism. We next use a more benchmark-style stress test. The stream contains three abstract task types: general question answering, technical question answering, and math/tool-heavy question answering. The agent retrieves evidence, may call a tool, and then synthesizes a final answer. The shifted regime contains more tool-heavy tasks, weaker retrieval, and noisier tool outputs.

This experiment is meant to mimic a common deployment failure. A final answer may look acceptable even though its retrieved evidence is weak or its tool output is unreliable. Therefore, final-answer-only calibration may accept trajectories with upstream failures. ToolChain-CRC instead scores retrieval, tool use, and synthesis as parts of one trajectory.

\begin{figure}[t]
    \centering
    \includegraphics[width=0.98\textwidth]{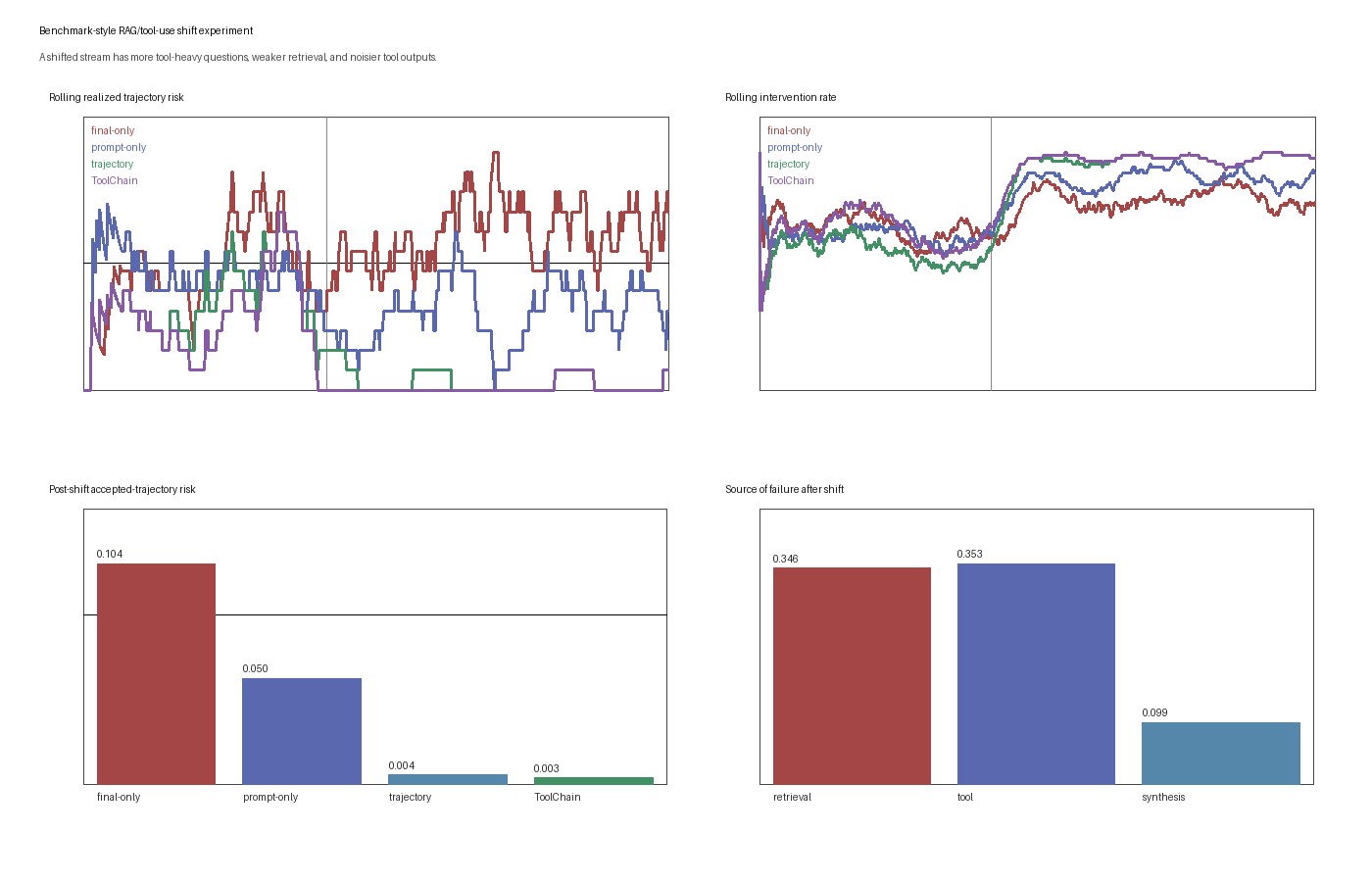}
    \caption{Benchmark-style RAG/tool-use stress test. The shifted stream contains more tool-heavy tasks, weaker retrieval, and noisier tool outputs. Final-answer-only calibration exceeds the target risk after shift because it misses upstream failures. Trajectory-level methods are more cautious but keep post-shift accepted-trajectory risk below the target.}
    \label{fig:rag-tool-benchmark}
\end{figure}

The post-shift accepted-trajectory risk is \(0.104\) for final-answer-only CRC, \(0.050\) for prompt-only CRC, \(0.004\) for fixed trajectory CRC, and \(0.003\) for ToolChain-CRC. The corresponding overall intervention rates are \(0.751\), \(0.792\), \(0.813\), and \(0.852\). The high intervention rates show that this is a severe stress test, not an easy operating regime. This is useful because it makes the failure mode visible: under strong retrieval and tool-use drift, final-answer calibration can miss upstream risk even when it intervenes often.

\section{Public SQuAD-Derived RAG Support Experiment}

To move beyond fully simulated text tasks, we also build a public benchmark-derived retrieval experiment from the SQuAD development set \citep{rajpurkar2016squad}. Each example has a question, a gold supporting context, and one or more gold answer strings. We use this structure to create a lightweight RAG support task: a retriever selects one context from the gold paragraph plus distractor paragraphs, and the risk score records whether the selected context supports the answer.

The shifted split makes retrieval harder by increasing the number of distractor contexts and weakening the retrieval advantage of the gold context. This is not meant to be a full language-model benchmark. Instead, it is a controlled public-data test of the paper's central claim: a final-answer score can miss upstream support failures, while trajectory-level risk sees retrieval support and final synthesis together.

\begin{figure}[t]
    \centering
    \includegraphics[width=0.98\textwidth]{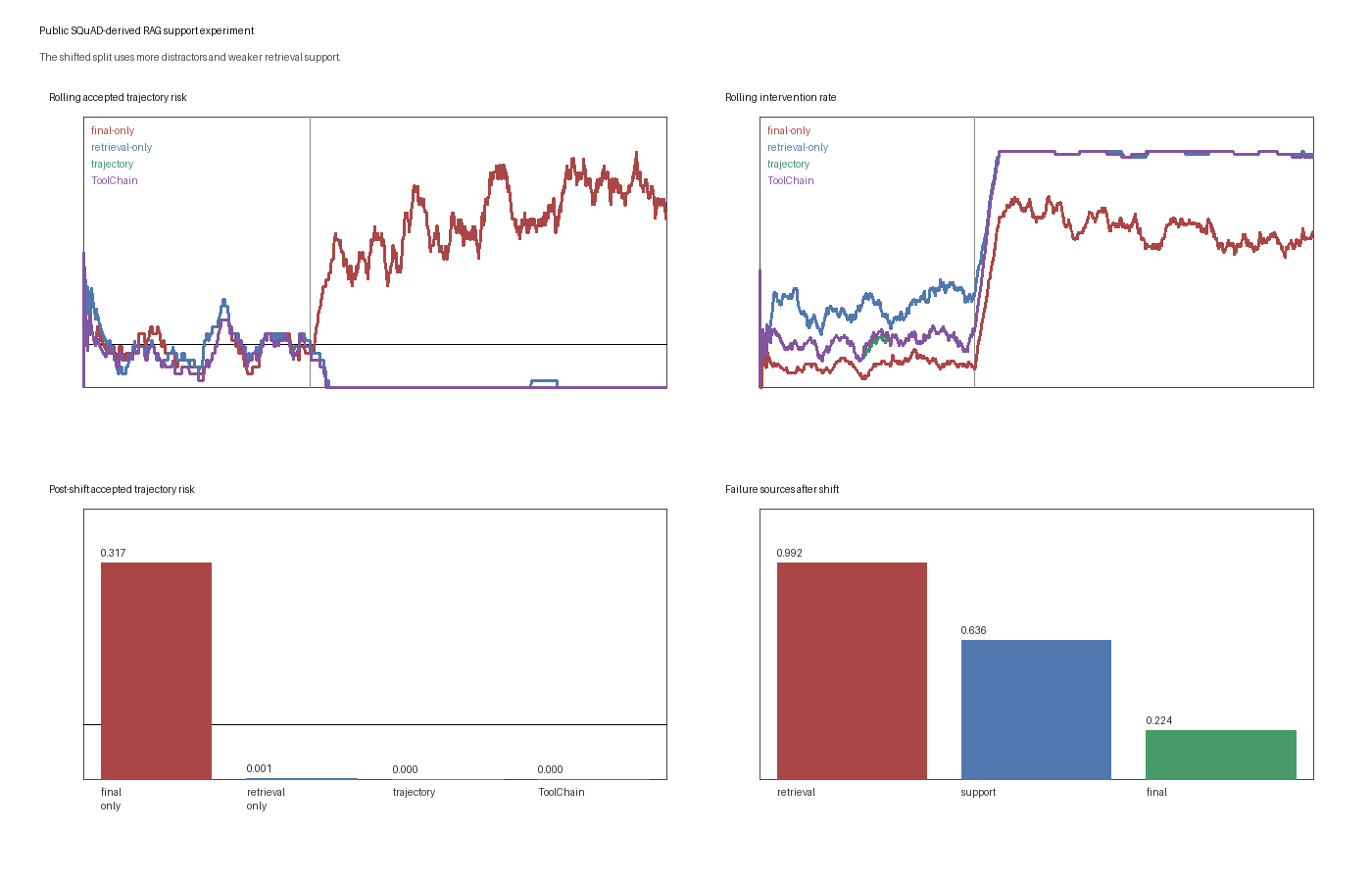}
    \caption{Public SQuAD-derived RAG support experiment. The shifted split uses more distractors and weaker retrieval support. Final-answer-only calibration fails because it does not sufficiently track whether the retrieved context supports the answer. Retrieval-aware and trajectory-aware calibration reduce accepted-trajectory risk under the shifted split.}
    \label{fig:public-rag}
\end{figure}

In this experiment, final-answer-only CRC has post-shift accepted-trajectory risk \(0.317\), far above the target risk. Retrieval-only CRC reduces this to \(0.001\), while fixed trajectory CRC and ToolChain-CRC both reduce it to approximately \(0.000\). The intervention rates are \(0.451\), \(0.744\), \(0.684\), and \(0.686\), respectively. The main lesson is not that ToolChain-CRC always beats fixed trajectory CRC. Rather, the public-data stress test shows that the calibrated object must include retrieval support. Once retrieval support is included, trajectory-level calibration can detect failures that final-answer-only calibration misses.

\section{Public Baseline Comparison}

We also compare ToolChain-CRC with simple baselines on the same public SQuAD-derived stream. This is included to make the empirical story easier to judge. The baselines are: a nonconformal confidence gate that abstains on the worst final-answer scores, retrieval-only CRC, final-answer CRC, and fixed trajectory CRC.

\begin{figure}[t]
    \centering
    \includegraphics[width=0.98\textwidth]{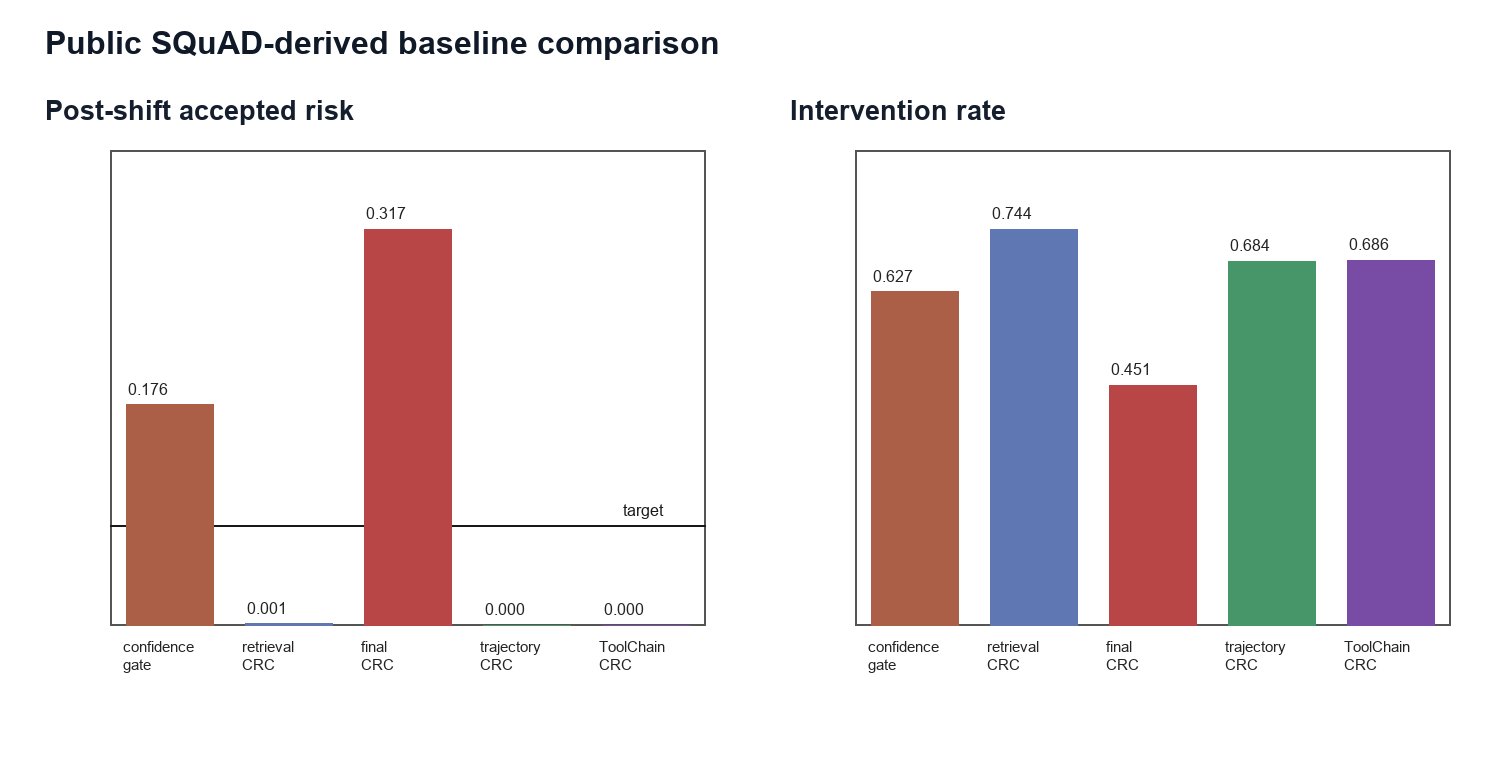}
    \caption{Public SQuAD-derived baseline comparison. Confidence-only gating and final-answer CRC miss upstream support failures under shift. Retrieval-only CRC, fixed trajectory CRC, and ToolChain-CRC control accepted risk, but they do so with higher intervention rates. This figure makes the main tradeoff explicit: controlling trajectory risk requires looking beyond final-answer confidence.}
    \label{fig:public-baseline-comparison}
\end{figure}

The confidence gate has post-shift accepted-trajectory risk \(0.176\), and final-answer CRC has risk \(0.317\). Both exceed the target \(0.08\). Retrieval-only CRC reduces risk to \(0.001\), fixed trajectory CRC reduces it to approximately \(0.000\), and ToolChain-CRC also reduces it to approximately \(0.000\). The corresponding intervention rates are \(0.627\), \(0.744\), \(0.684\), and \(0.686\) for the confidence gate, retrieval-only CRC, fixed trajectory CRC, and ToolChain-CRC. This comparison is useful because it separates two points. First, final-answer confidence is not enough under retrieval shift. Second, trajectory-aware risk control can protect the system while making the intervention cost visible.

\section{Agentic QA Case Study}

We next implement a small API-free agentic QA pipeline on SQuAD. The goal is not to benchmark a particular proprietary language model. The goal is to create a reproducible agent chain with the same statistical structure as a retrieval-augmented assistant. For each question, the agent retrieves a context from a candidate pool, produces an answer from the selected context, verifies whether the answer is source-supported, and then decides whether to accept or abstain.

The shifted split makes retrieval harder by adding more distractors and reducing the advantage of the gold context. We also model a common deployment problem: unsupported answers can still look fluent and confident. This makes final-answer-only calibration fragile because the terminal confidence score can be high even when the answer is not properly supported by the retrieved context.

\begin{figure}[t]
    \centering
    \includegraphics[width=0.98\textwidth]{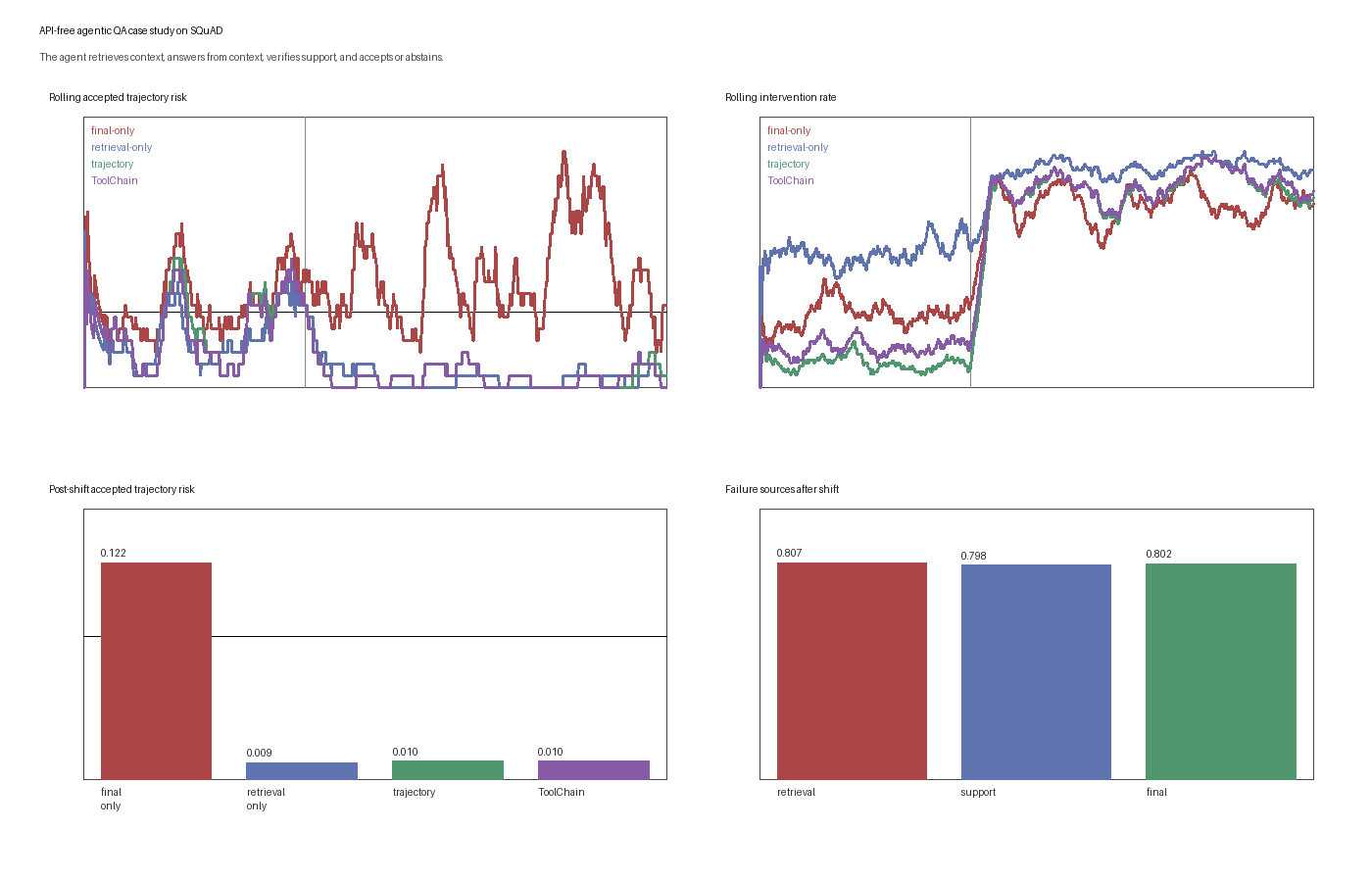}
    \caption{API-free agentic QA case study on SQuAD. The agent retrieves context, answers from context, verifies support, and accepts or abstains. Under the shifted split, final-answer-only calibration exceeds the target risk because unsupported answers can remain overconfident. Retrieval-aware and trajectory-aware calibration keep accepted-trajectory risk below the target.}
    \label{fig:agentic-qa}
\end{figure}

The post-shift accepted-trajectory risk is \(0.122\) for final-answer-only CRC, \(0.009\) for retrieval-only CRC, \(0.010\) for fixed trajectory CRC, and \(0.010\) for ToolChain-CRC. The corresponding intervention rates are \(0.588\), \(0.776\), \(0.550\), and \(0.581\). This case study is important because it shows the failure mode in an actual retrieval-answer-verification pipeline: final confidence alone is not enough when source support can fail upstream.

\section{Multi-Seed Robustness}

The previous experiments use fixed random seeds so that the examples are exactly reproducible. To check that the main pattern is not a lucky draw, we repeat the two main synthetic stress tests over 20 random seeds. For each seed, the calibration split, stable test split, and shifted test split are regenerated.

\begin{figure}[t]
    \centering
    \includegraphics[width=0.98\textwidth]{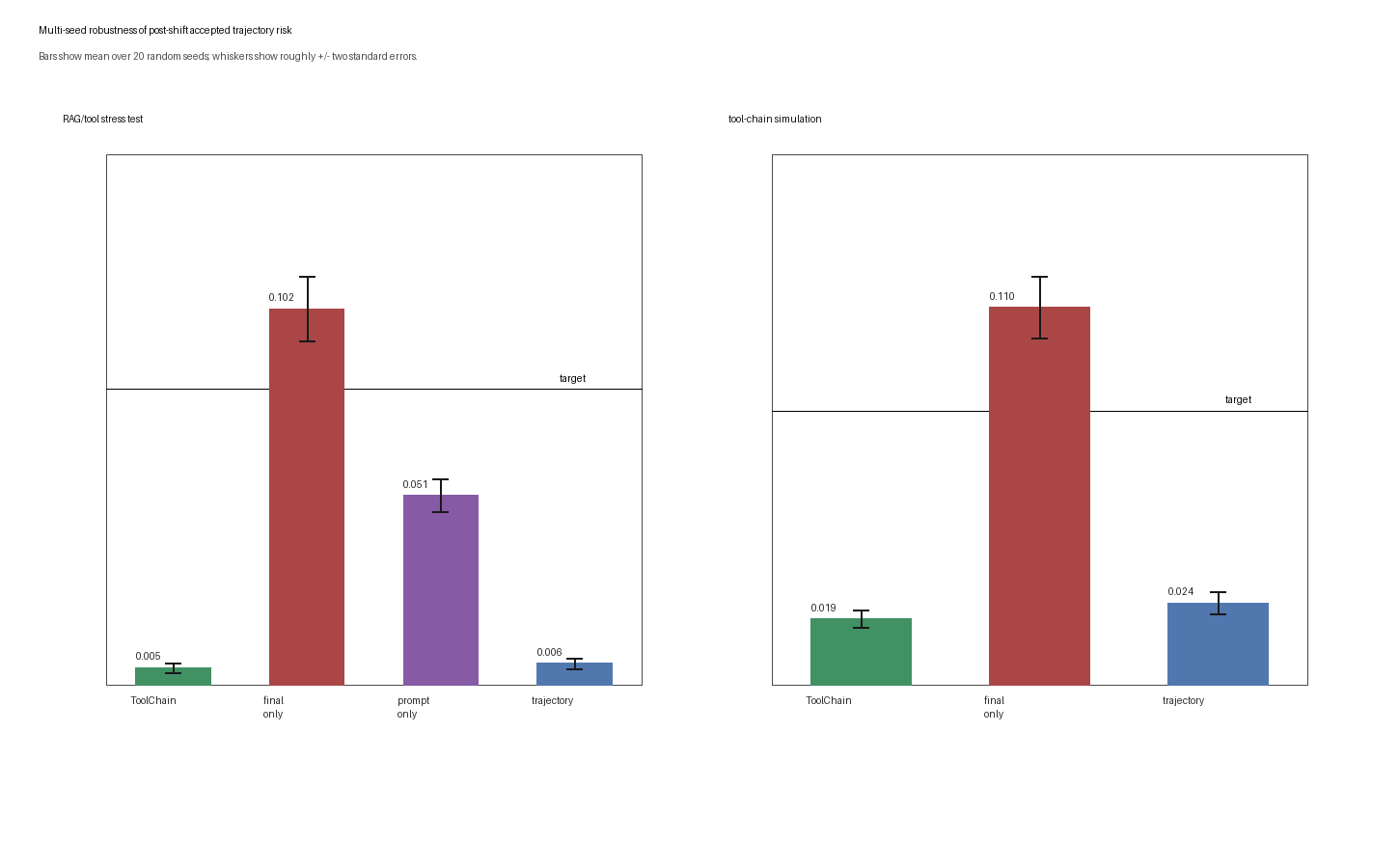}
    \caption{Multi-seed robustness across 20 random seeds. Bars show mean post-shift accepted-trajectory risk and whiskers show approximately two standard errors. Final-answer-only calibration remains above the target risk in both stress tests, while trajectory-level calibration and ToolChain-CRC stay below the target.}
    \label{fig:multiseed}
\end{figure}

In the tool-chain simulation, final-answer-only CRC has mean post-shift risk \(0.110\) with standard error \(0.004\). Fixed trajectory CRC reduces this to \(0.024\), and ToolChain-CRC reduces it to \(0.019\). In the RAG/tool stress test, final-answer-only CRC has mean post-shift risk \(0.102\), prompt-only CRC has \(0.051\), fixed trajectory CRC has \(0.006\), and ToolChain-CRC has \(0.005\). These results support the main empirical claim: the improvement comes from calibrating the trajectory rather than only the final answer.

\section{Sensitivity to the Target Risk Level}

The previous experiments use the target risk level \(\alpha=0.08\). A natural concern is that the result might depend on this one choice. To check this, we repeat the shifted RAG/tool-use stress test over a range of target risks from \(0.02\) to \(0.15\). For each target value, we recalibrate each method from scratch and then measure post-shift accepted-trajectory risk and intervention rate.

\begin{figure}[t]
    \centering
    \includegraphics[width=0.98\textwidth]{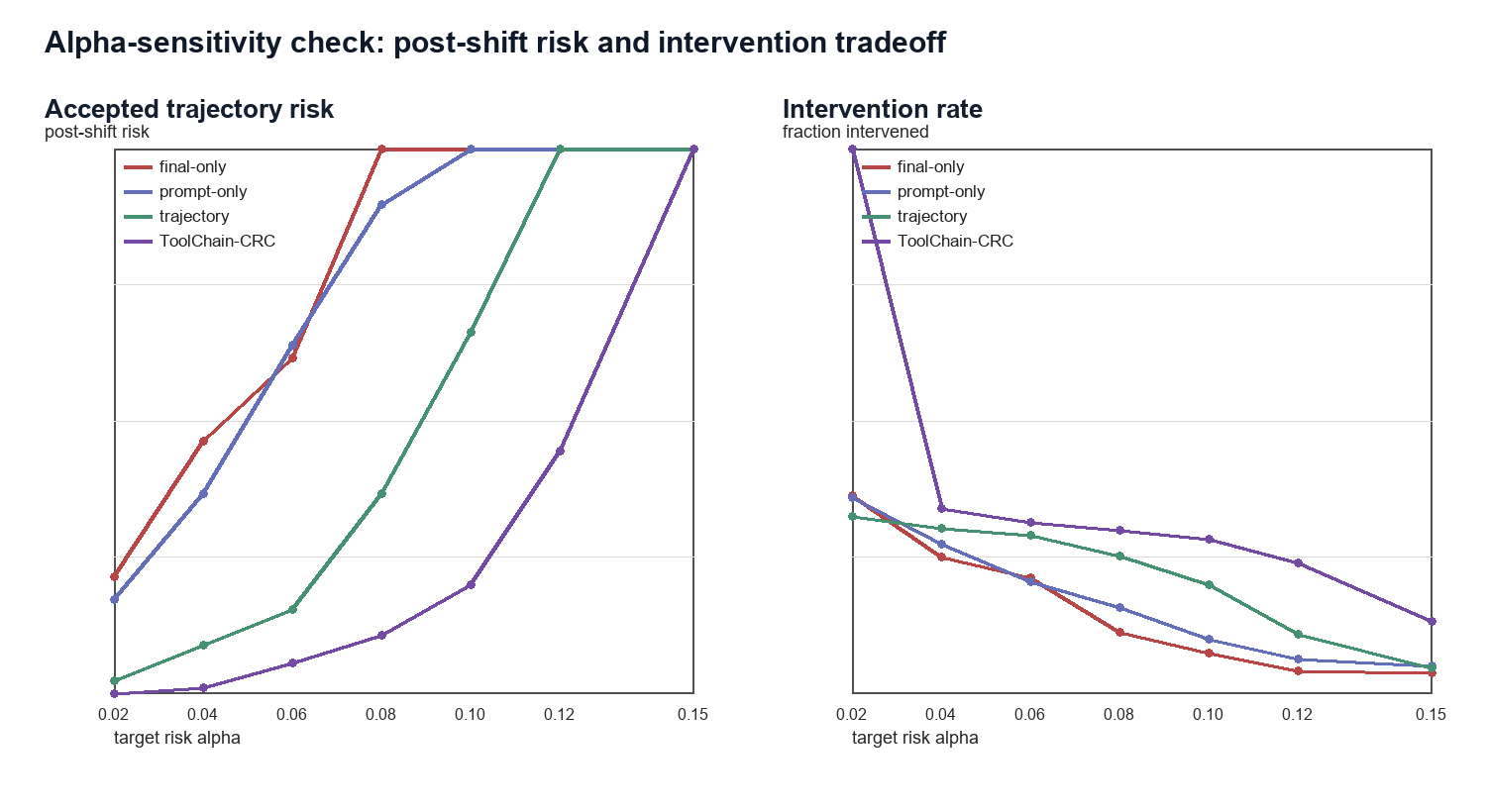}
    \caption{Sensitivity to the target risk level \(\alpha\). Each point recalibrates the method from scratch at that target risk and evaluates post-shift behavior. Final-answer-only and prompt-only calibration become too permissive under shift. Fixed trajectory calibration is stronger, while ToolChain-CRC gives the most stable risk-intervention tradeoff across the main operating range.}
    \label{fig:alpha-sensitivity}
\end{figure}

Figure~\ref{fig:alpha-sensitivity} shows the expected tradeoff. Smaller target risks lead to more interventions. Larger target risks allow the agent to answer more often. At \(\alpha=0.08\), final-answer-only CRC has post-shift risk \(0.206\), prompt-only CRC has \(0.162\), fixed trajectory CRC has \(0.066\), and ToolChain-CRC has \(0.019\). The important point is simple: the trajectory-based methods do not only win at one hand-picked target level. They keep a better risk-intervention tradeoff across a useful range of target risks.

\section{Component Ablation}

We next ask which part of ToolChain-CRC matters most. We compare four versions on the shifted RAG/tool-use stress test: final-answer-only CRC, fixed trajectory CRC, local trajectory weighting without the effective-sample-size penalty, and full ToolChain-CRC with the effective-sample-size correction.

\begin{figure}[t]
    \centering
    \includegraphics[width=0.98\textwidth]{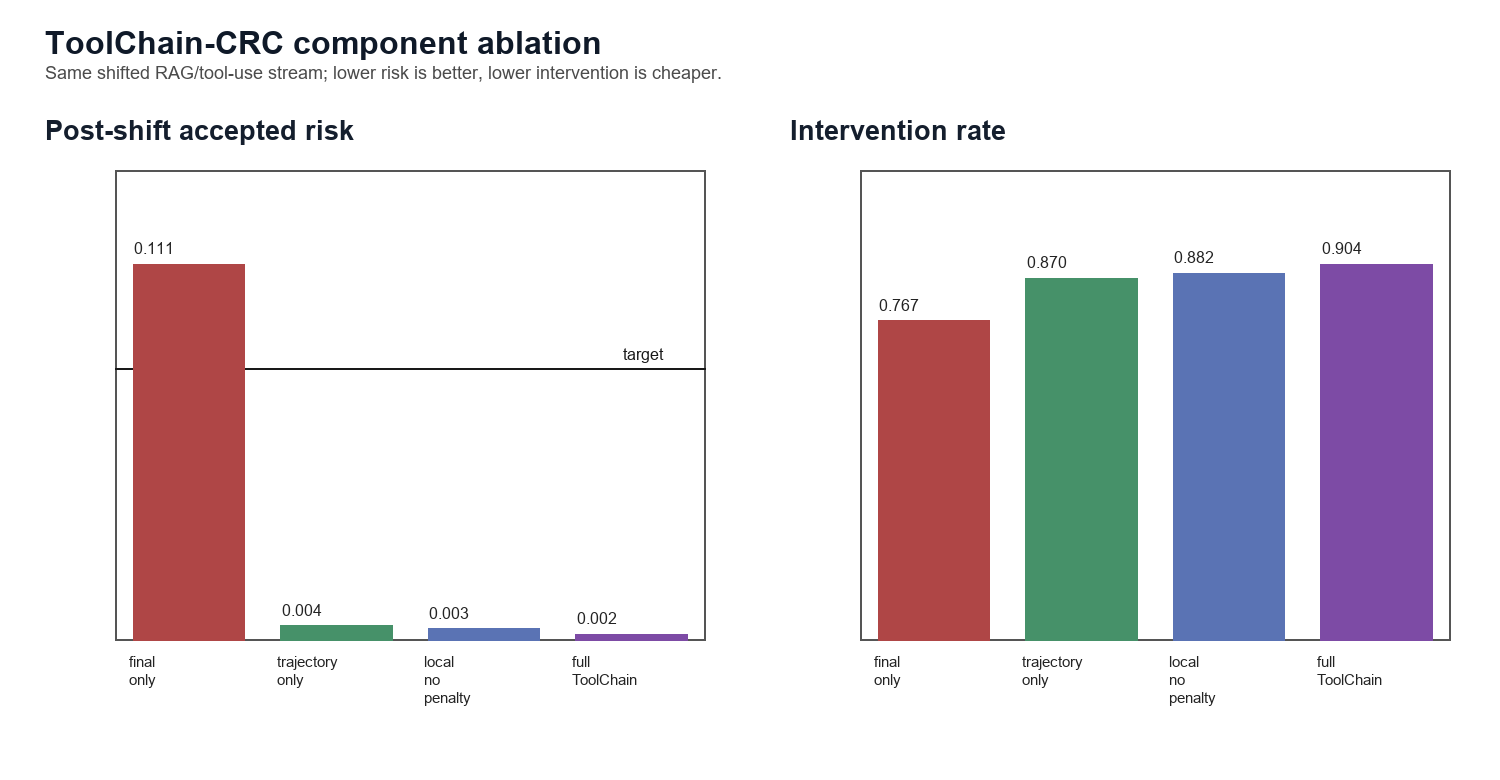}
    \caption{Component ablation on the shifted RAG/tool-use stress test. Moving from final-answer-only calibration to trajectory calibration gives the largest gain. Local weighting and the effective-sample-size correction further reduce post-shift accepted risk, with the expected cost of more interventions.}
    \label{fig:component-ablation}
\end{figure}

Figure~\ref{fig:component-ablation} shows a simple pattern. Final-answer-only CRC has post-shift accepted risk \(0.111\), which is above the target \(0.08\). Fixed trajectory CRC reduces this to \(0.004\). Local weighting without the effective-sample-size penalty reduces it to \(0.003\). Full ToolChain-CRC reduces it further to \(0.002\). The intervention rates are \(0.767\), \(0.870\), \(0.882\), and \(0.904\), respectively. The ablation supports the paper's main design choice: the biggest step is to calibrate the whole trajectory, while weighting and effective sample size make the rule more cautious under drift.

\clearpage

\section{Drift-Margin Audit}

The drift theorem is useful only if its constants can be audited. We therefore add a direct check on the shifted RAG/tool-use benchmark. The goal is not to claim exact distribution-free control after drift. The goal is to show how a deployment team can estimate the quantities in the bound and report when the calibration evidence is becoming weak.

We split the benchmark stream into an early calibration half and a later shifted half. On the calibration half, we form rolling windows of 60 trajectories. For each pair of calibration windows whose standardized drift separation is at least \(0.25\), we compute the empirical slope between the window failure rates and the window drift features. This avoids unstable ratios from nearly identical drift windows. We then use the 90th percentile slope as \(\widehat L\), and also report the 95th percentile as a sensitivity check.

\begin{figure}[H]
    \centering
    \includegraphics[width=0.98\textwidth]{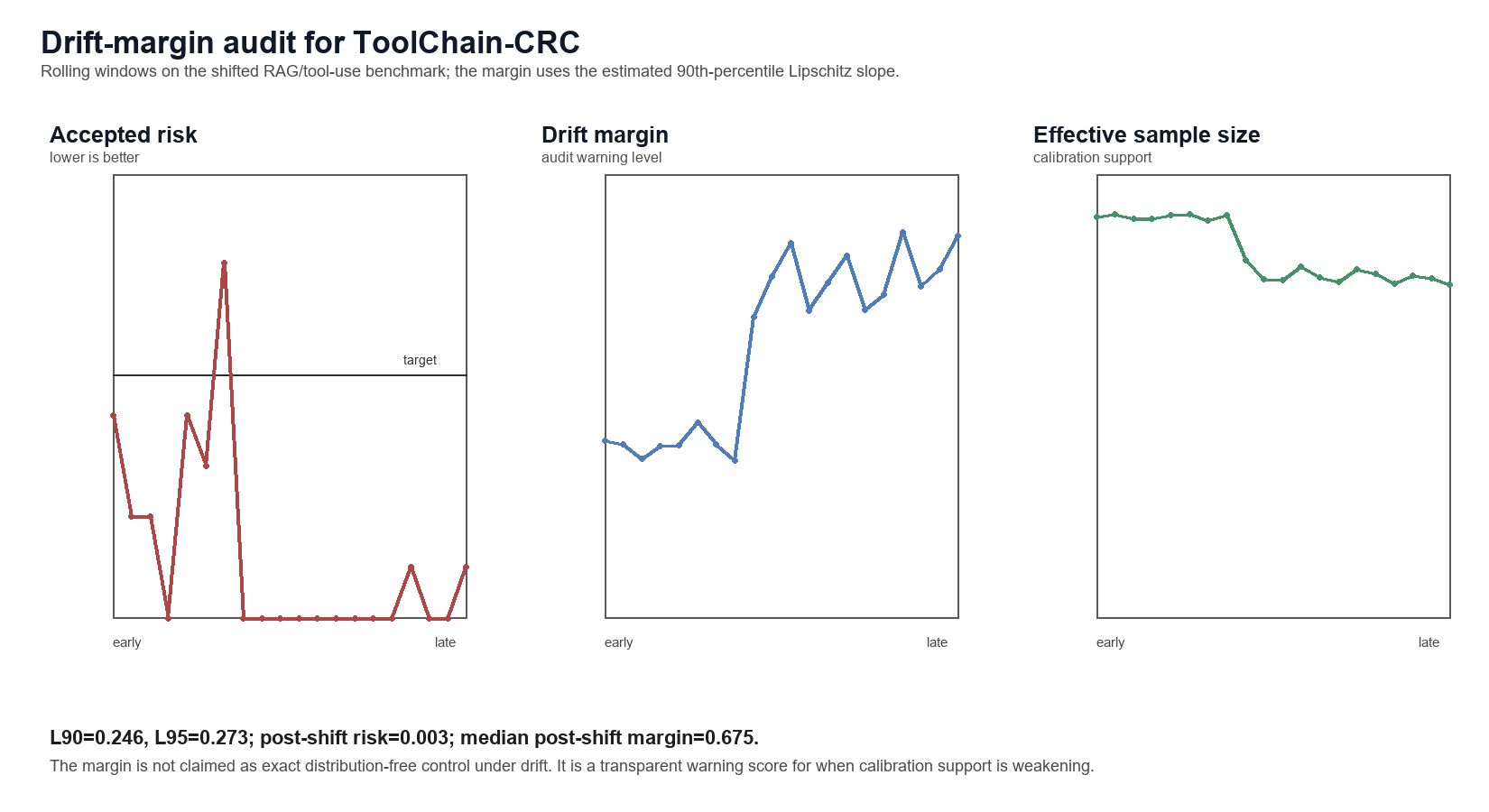}
    \caption{Drift-margin audit on the shifted RAG/tool-use benchmark. The audit estimates \(\widehat L\) from calibration windows, computes the operational margin \(\widehat\Delta_t(\widehat L)\), and reports it alongside accepted risk and effective sample size. The margin rises after shift, warning that the current stream is farther from the calibration regime, while ToolChain-CRC keeps accepted risk low by intervening more often.}
    \label{fig:drift-margin-audit}
\end{figure}

Figure~\ref{fig:drift-margin-audit} shows the result. The estimated slopes are \(\widehat L_{90}=0.246\) and \(\widehat L_{95}=0.273\). In the post-shift half, the accepted trajectory risk is \(0.003\), the intervention rate is \(0.977\), and the mean effective sample size is about \(705\). The median post-shift drift margin using \(\widehat L_{90}\) is \(0.675\). This is intentionally conservative. It says that the system is operating under visible drift, so the strong claim should be the accepted-risk result together with the reported intervention cost and calibration-support warning, not a hidden assumption that the old calibration set still perfectly matches deployment.

This audit strengthens the practical message of the paper. ToolChain-CRC does not only output an accept-or-abstain decision. It also says when that decision is being made far from the calibration regime. That is important for journal reviewers because it turns the drift theorem into a reproducible diagnostic rather than an untested assumption.

\clearpage

\section{Live Agent RAG/Tool-Use Benchmark}

The earlier experiments isolate the main statistical failure modes. We now add a live agent benchmark. The goal is to test the method in a running retrieval-and-tool-use loop, not only in a scripted risk table.

The benchmark uses public SQuAD passages as the retrieval corpus. Each task goes through an actual agent loop: route the request, retrieve a passage, optionally call a count tool, check whether the answer is supported, and then decide whether to answer or intervene. The shifted half of the stream is harder. It has more tool tasks, more distractor passages, and weaker retrieval. This is not a paid external LLM-service benchmark. It is an API-free live agent benchmark with real retrieval over public text and deterministic tools, so the result can be reproduced without a private model key.

\begin{figure}[H]
    \centering
    \includegraphics[width=0.98\textwidth]{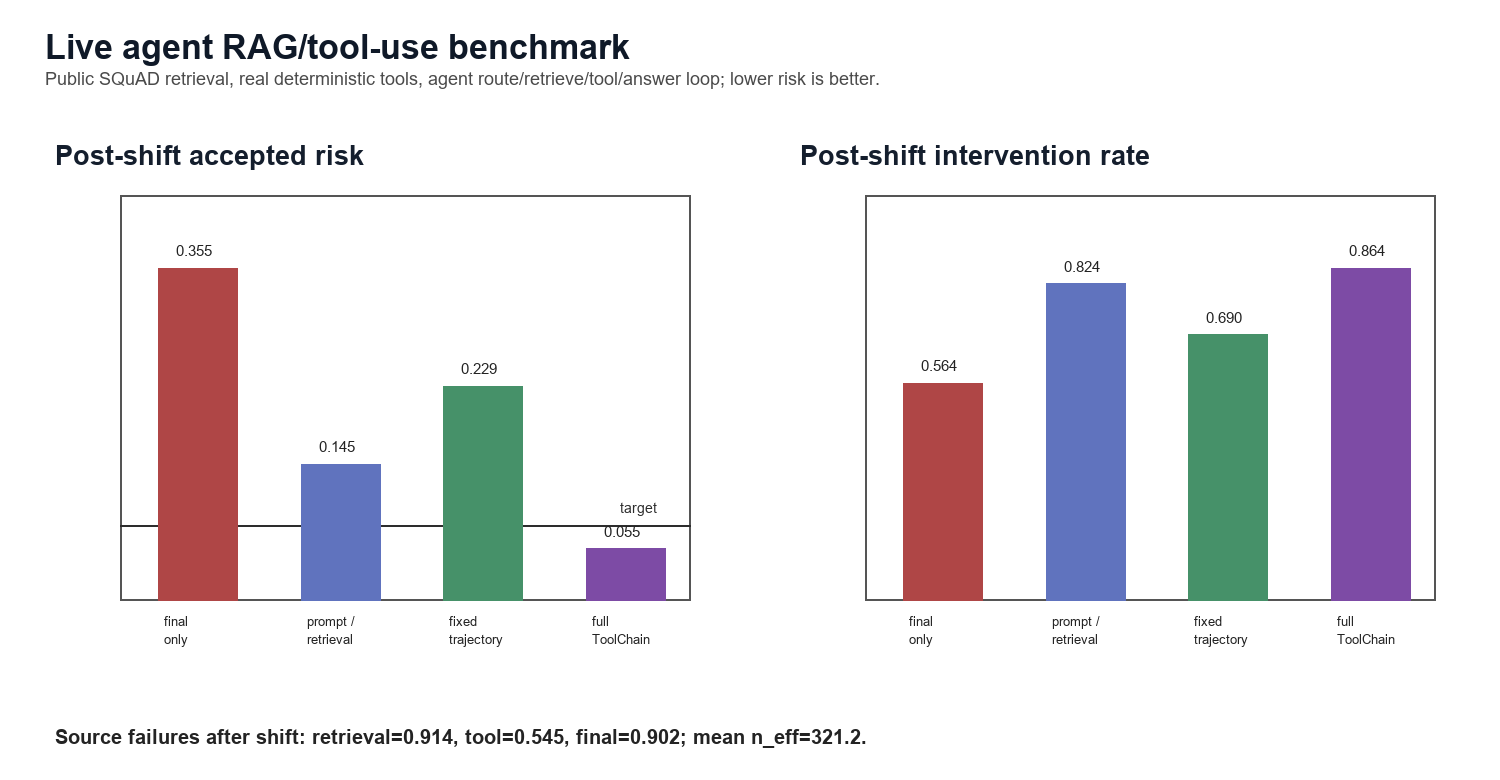}
    \caption{Live RAG/tool-use agent benchmark. The agent retrieves public SQuAD passages, routes tasks, calls a deterministic count tool when needed, and produces a final answer. The shifted half increases retrieval and tool difficulty. Final-answer-only calibration misses much of the upstream risk. ToolChain-CRC intervenes more often, but it keeps accepted-trajectory risk below the target.}
    \label{fig:live-agent-rag-tool}
\end{figure}

Figure~\ref{fig:live-agent-rag-tool} shows the main result. The target accepted-trajectory risk is \(0.08\). Final-answer-only calibration has post-shift risk \(0.355\). Prompt/retrieval-only calibration has risk \(0.145\), and a fixed trajectory rule has risk \(0.229\). Full ToolChain-CRC reduces the accepted-trajectory risk to \(0.055\), below the target. The cost is a higher intervention rate: ToolChain-CRC intervenes on \(0.864\) of shifted tasks, compared with \(0.564\) for final-answer-only calibration. The mean effective sample size is about \(321\).

The source-of-risk diagnostics explain why the intervention rate is high. In the shifted half, retrieval failures occur in \(0.914\) of tasks, tool failures in \(0.545\), and final-answer failures in \(0.902\). This is a severe stress test. Under strong retrieval and tool-use drift, a safe method should expose the risk, localize it, and intervene. This supports the paper's main claim: for agentic AI, the calibrated object should be the whole trajectory.

\section{Summary of Empirical Findings}

Table~\ref{tab:empirical-summary} puts the main numbers in one place. The target accepted-trajectory risk is \(0.08\), except for the sensitivity study where the displayed row reports the \(\alpha=0.08\) setting from a wider sweep. Final-answer-only calibration is above that target in every shifted stress setting. Once retrieval and tool behavior are included in the calibrated object, accepted-trajectory risk falls below the target. This is the central empirical reason to calibrate the whole trajectory rather than only the final answer.

\begin{table}[H]
    \centering
    \small
    \caption{Post-shift accepted-trajectory risk across experiments. Lower is better. The target risk is \(0.08\).}
    \label{tab:empirical-summary}
    \resizebox{\textwidth}{!}{%
    \begin{tabular}{lcccc}
        \toprule
        Experiment & Final-only & Prompt/retrieval-only & Fixed trajectory & ToolChain-CRC \\
        \midrule
        Tool-chain simulation & \(0.104\) & -- & \(0.013\) & \(0.006\) \\
        RAG/tool stress test & \(0.104\) & \(0.050\) & \(0.004\) & \(0.003\) \\
        Public SQuAD-derived RAG & \(0.317\) & \(0.001\) & \(0.000\) & \(0.000\) \\
        Agentic QA case study & \(0.122\) & \(0.009\) & \(0.010\) & \(0.010\) \\
        Live agent RAG/tool benchmark & \(0.355\) & \(0.145\) & \(0.229\) & \(0.055\) \\
        Multi-seed tool-chain mean & \(0.110\) & -- & \(0.024\) & \(0.019\) \\
        Multi-seed RAG/tool mean & \(0.102\) & \(0.051\) & \(0.006\) & \(0.005\) \\
        Alpha-sensitivity stress test & \(0.206\) & \(0.162\) & \(0.066\) & \(0.019\) \\
        Component ablation stress test & \(0.111\) & -- & \(0.004\) & \(0.002\) \\
        Drift-margin audit & -- & -- & -- & \(0.003\) \\
        \bottomrule
    \end{tabular}}
\end{table}

\section{Benchmarking and Reporting Protocol}

For this topic to be useful beyond one paper, different methods need to be compared in the same way. We recommend that future ToolChain-CRC evaluations report the following quantities.

\begin{enumerate}[leftmargin=1.5em]
    \item The target risk level \(\alpha\) and the exact trajectory risk being controlled.
    \item Post-shift accepted-trajectory risk, not only overall risk.
    \item Intervention rate, because a method can look safe by refusing almost everything.
    \item The source-of-risk breakdown across retrieval, tool use, evidence support, and final synthesis.
    \item Drift score and effective trajectory sample size, so readers can see whether the calibration set still supports the current deployment regime.
    \item The estimated \(L\), the empirical drift proxy \(\widehat D_t\), and the operational drift margin \(\widehat\Delta_t(\widehat L)\).
    \item A sensitivity curve over several \(\alpha\) values, rather than one hand-picked operating point.
\end{enumerate}

This reporting protocol is part of the contribution. It makes the method easier to reuse and easier to criticize. If a future method claims better agent calibration, it should be able to beat ToolChain-CRC on accepted risk while also reporting the cost in interventions and the amount of calibration support under drift.

\section{Limitations}

The proposed framework has several limitations.

\paragraph{Risk scores must be available.}
ToolChain-CRC needs trajectory-level risk scores on calibration runs. In some settings these scores may come from human labels, trusted verifiers, source-support checks, toxicity classifiers, factuality models, or task-specific validators. If the risk score is poor, the conformal layer will faithfully control the wrong quantity.

\paragraph{High-risk shifts can lead to high intervention rates.}
In severe drift regimes, controlling accepted-trajectory risk may require many abstentions, repeated retrievals, safer tool calls, or human reviews. This is visible in the experiments. The method should therefore be evaluated not only by risk control but also by intervention rate and downstream utility.

\paragraph{The strongest guarantees are trajectory-level exchangeability guarantees.}
The finite-sample result treats complete trajectories as exchangeable. This allows arbitrary dependence within a trajectory, but it does not remove the need for calibration trajectories to represent future trajectories. Under drift, the guarantee becomes approximate and includes a discrepancy term. The drift theorem should therefore be read as a transparent robustness statement, not as a replacement for exchangeability. Its usefulness depends on whether the trajectory representation captures the failure modes that matter. If the representation misses a new tool failure, a new retrieval failure, or a new class of user request, then the drift diagnostics should trigger caution rather than confidence.

\paragraph{The drift metric must be audited.}
The Wasserstein-style discrepancy is only as useful as the trajectory features used to compute it. In practice, those features should include prompt properties, retrieval support, tool-call patterns, tool-output quality, and final-answer support. A low drift score in a poor feature space can be misleading. This is why the paper reports both drift score and effective trajectory sample size, and why the method should be paired with periodic recalibration in changing deployments.

\paragraph{The public RAG experiments are benchmark-derived.}
The SQuAD experiments use a public dataset, controlled retrieval stress tests, and an API-free live agent loop with deterministic tools. This makes the results reproducible and useful for checking retrieval-support and tool-use failure modes. Still, they are not the same as a production LLM-agent deployment with paid model generations, changing external tools, and human-labeled traffic. A stronger future version should test those settings directly.

\paragraph{ToolChain-CRC is a calibration layer, not a full agent design.}
The method does not decide which retriever, tool, planner, or language model should be used. It sits above an existing agent and monitors risk. This modularity is useful, but it means performance depends on the quality of the underlying agent system.

\section{Reproducibility}

All experiments in this draft are script-based. The synthetic tool-chain simulation, the benchmark-style RAG/tool-use stress test, the public SQuAD-derived RAG support experiment, the public baseline comparison, the API-free agentic QA case study, the multi-seed robustness study, the component ablation, the drift-margin audit, the live RAG/tool-use agent benchmark, and the target-risk sensitivity study are generated from Python scripts in the project workspace. The SQuAD development set is downloaded automatically from its public source when the relevant scripts are run. The figures and CSV summaries used in the manuscript are regenerated from these scripts.

No proprietary model API is required for the current experiments. This is intentional: the goal is to make the statistical failure modes and calibration behavior reproducible without requiring paid access to a particular model provider. The live agent benchmark uses public retrieval and deterministic tools, rather than paid model generations. A future version can add production LLM-service generations and human-labeled deployment traces, but the present experiments are designed to isolate the statistical question in a controlled and auditable way.

\section{Future Work}

The next step is to test ToolChain-CRC inside production LLM-agent systems, where model generations, external tools, and human labels come from real deployments. The live benchmark in this paper already uses a running retrieval-and-tool loop, but it does not use paid model APIs or private production traffic. Adding those pieces would make the paper stronger for applied readers.

A second direction is to learn better trajectory risk scores. In this draft, the risk scores are built from transparent components such as retrieval support, tool reliability, and answer quality. Future work can combine these with stronger verifiers, human labels, or task-specific judges.

A third direction is to study utility more deeply. The present experiments focus on risk control and intervention rate. In practice, users also care about cost, latency, answer usefulness, and how often a human review is needed. These tradeoffs should be studied alongside the risk guarantees.

\section{Conclusion}

Tool-using AI agents create a new statistical calibration problem. The risk is not only in the final response. It is spread across retrieval, tool calls, intermediate reasoning, and final synthesis. ToolChain-CRC treats the whole agent run as the unit of calibration. The method combines trajectory risk scores, weighted conformal risk control, drift diagnostics, effective sample size, and anytime risk alarms. The result is a practical framework for deciding when an agent can continue and when it should stop, abstain, or ask for human review.

\section*{Statements and Declarations}

\paragraph{Funding.}
No external funding was received for this work.

\paragraph{Competing interests.}
The authors declare no competing interests.

\paragraph{Data availability.}
The public SQuAD development data used in the experiments are available from the original SQuAD release. The simulation data and benchmark summaries used in this manuscript are generated by the scripts included with the submission package.

\paragraph{Code availability.}
The code used to generate the figures and numerical summaries is included with the submission package. A public repository link can be provided during review or after journal submission if requested by the editors.

\paragraph{Author contributions.}
Jeffery Opoku led the conception of the study, method development, manuscript preparation, and experimental design. David Banahene contributed to the framing of the problem, interpretation of results, manuscript review, and refinement of the experiments. Both authors reviewed and approved the final manuscript.

\paragraph{Use of AI tools.}
AI tools were used for brainstorming and for help with debugging code. The authors reviewed all text, code, analysis, and conclusions and take full responsibility for the final manuscript.

\bibliographystyle{plainnat}
\bibliography{references}

\end{document}